\documentclass{article}

\usepackage{PRIMEarxiv}

\usepackage[utf8]{inputenc} 
\usepackage[T1]{fontenc}    
\usepackage{url}            
\usepackage{booktabs}       
\usepackage{amsfonts}       
\usepackage{nicefrac}       
\usepackage{microtype}      
\usepackage{fancyhdr}       
\usepackage{graphicx}       
\usepackage{soul}
\usepackage{xcolor}
\usepackage{times}
\usepackage{amsmath}
\usepackage{amssymb}
\usepackage{float}
\usepackage{subcaption}
\usepackage{varwidth}
\usepackage{adjustbox}
\usepackage{arydshln}
\usepackage{enumitem}
\usepackage[section]{placeins}

\usepackage{amsthm}
\usepackage{colortbl}
\usepackage{multirow}
\usepackage{algorithm}
\usepackage{algpseudocode}
\newtheorem{proposition}{Proposition}

\graphicspath{{media/}}     
\restylefloat{table}

\DeclareMathOperator*{\argmin}{arg\,min}
\title{F3: Fair and Federated Face Attribute Classification with Heterogeneous Data}

\author{
  Samhita Kanaparthy, Manisha Padala, Sankarshan Damle  \\
  Machine Learning Lab, IIIT Hyderabad\\
  Hyderabad, India \\
  \texttt{(s.v.samhita, manisha.padala, sankarshan.damle)@research.iiit.ac.in} \\
   \And
  Ravi Kiran Sarvadevabhatla\\
  CVIT Lab, IIIT Hyderabad\\
  Hyderabad, India \\
  \texttt{ravi.kiran@iiit.ac.in} \\\And
  Sujit Gujar\\
  Machine Learning Lab, IIIT Hyderabad\\
  Hyderabad, India \\
  \texttt{sujit.gujar@iiit.ac.in} \\
}

\begin{document}
\maketitle

\begin{abstract}

Fairness across different demographic groups is an essential criterion for face-related tasks, Face Attribute Classification (FAC) being a prominent example. Apart from this trend, Federated Learning (FL) is increasingly gaining traction as a scalable paradigm for distributed training. Existing FL approaches require data homogeneity to ensure fairness. However, this assumption is too restrictive in real-world settings. We propose F3, a novel FL framework for fair FAC under data heterogeneity. F3 adopts multiple heuristics to improve fairness across different demographic groups without requiring data homogeneity assumption. We demonstrate the efficacy of F3 by reporting empirically observed fairness measures and accuracy guarantees on popular face datasets. Our results suggest that F3 strikes a practical balance between accuracy and fairness for FAC.

\end{abstract}

\section{Introduction}

\label{sec:intro}

{Face Attribute Classification} (FAC) finds prominence for tasks such as gender classification~\cite{gender}, face verification~\cite{face1}, and face identification~\cite{face2}. Recently, researchers highlight a critical issue in FAC: attribute prediction may be biased towards specific demographic groups. For instance, \cite{buolamwini2018gender} show that for the gender classification task on the MSFT dataset, the error rate for `darker' faces is approximately eighteen times greater than that on `lighter' faces\footnote{``Facial Recognition is accurate, if you’re a white guy", New York Times, 2018.}. Further, face recognition-based criminal detection systems are prone to classify innocent people with `darker' faces as suspects~\cite{criminalbias}. This bias in predictions is \emph{unfairness}. It is often associated with the unavailability of balanced datasets~\cite{torralba2011unbiased}. To overcome this issue, researchers have introduced balanced, large-scale datasets~\cite{fairface}.

Of late, Federated Learning (FL) has emerged as a popular paradigm for scalable distributed training involving large-scale data~\cite{google-fedml}.  FL comprises (i) independent clients that train local models on their private data and (ii) a central aggregator which combines these local models, using heuristics, to derive a generalisable global model~\cite{aggr}. Unfortunately, traditional FL models typically focus on standard performance measures (e.g., accuracy) and inherit the unfairness related drawbacks of non-FL approaches~\cite{FedCV}.

%
\begin{figure}[t]
    \centering
    \includegraphics[width=0.9\linewidth]{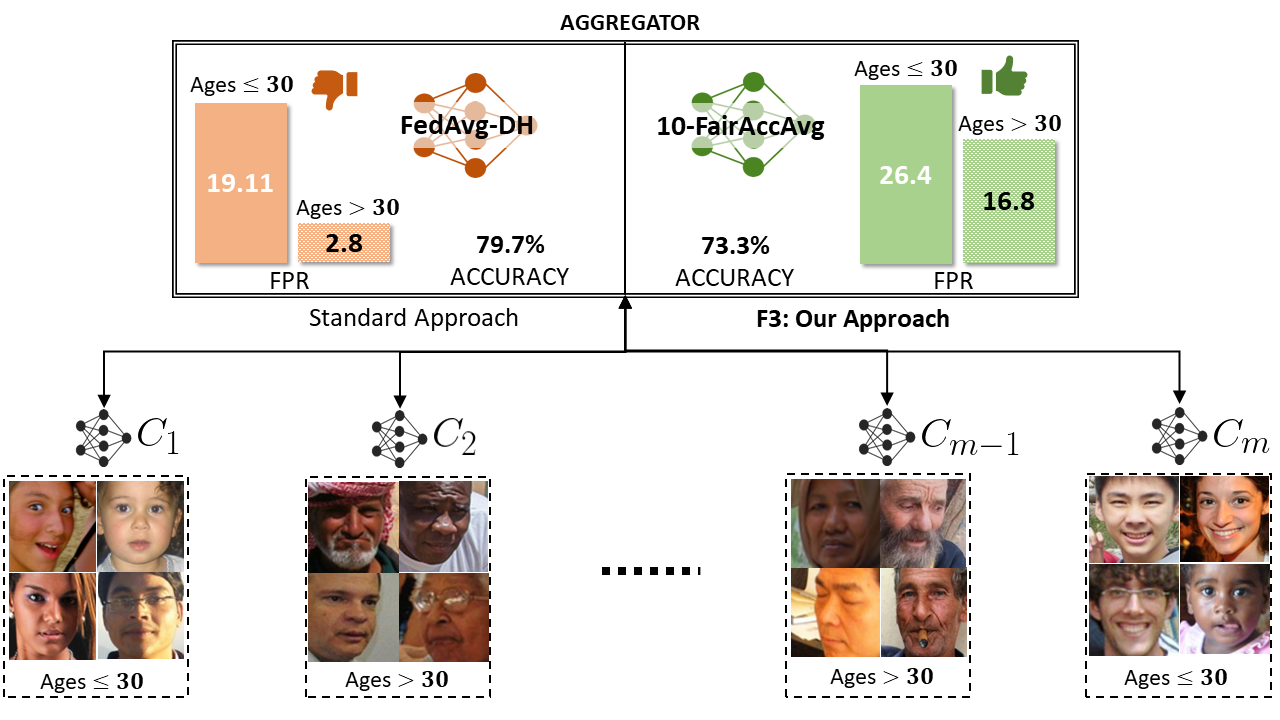}
    \caption{Our F3 framework considers FL with Data Heterogeneity for Fair Face Attribute Classification. With age as a sensitive attribute, each client ($C_1,\ldots,C_m$) has access to samples of only a single age group. The difference in error rate (FPR) observed for the two age groups ($\leq 30$ and $> 30$) is high for the state-of-the-art approach, \texttt{FedAvg-DH}. In contrast, our novel heuristic, 10-\textsf{FairAccAvg}, noticeably improves fairness (i.e., gap between error rates across age groups) while maintaining reasonable accuracy.}
    \label{fig:teaser}
\end{figure}
%
To address the unfairness in FL many methods have been introduced (e.g., \cite{ditto,Fair,fairfed}). However, these methods inherently assume FL clients with  homogeneous data, i.e, they assume that FL clients' data contains samples from all the demographic groups of a particular \emph{sensitive attribute}. E.g., with `age' as the sensitive attribute, the client's local training data would have samples from both `young' and `adult' demographic groups. However, clients' data is likely to be heterogeneous in most FL settings.

For instance, a smartphone belonging to a `young' user may have content comprising majorly to its peers~\cite{ruan2021fedsoft}, i.e., inter-client heterogeneity in terms of age. Similarly, geographically separated clients may exhibit inter-client heterogeneity in race. This data heterogeneity may, in turn reduce fairness for tasks such as FAC. Consequently, one \emph{cannot} directly adopt the existing approaches~\cite{agarwal18,Jung_2021_CVPR,fnnc} to improve fairness.

\smallskip
\noindent\textbf{Our Approach:} To address the shortcomings mentioned above, we propose \emph{F3}, a novel \underline{F}L framework for \underline{F}air \underline{F}ace Attribute Classification -- refer Fig.~\ref{fig:teaser}. F3 guarantees fairness by employing appropriate aggregation heuristics for the extreme case of heterogeneous face data in the FL setting. In particular, the following are our contributions:
 \begin{enumerate}[leftmargin=*,noitemsep]
     \item We are first to introduce and study fair Face Attribute Classification (FAC) in FL under data heterogeneity (FL with DH). We prove that existing approaches to ensure fairness are not applicable in this setting (Proposition~\ref{prop::inf}).
    
    \item To incorporate fairness in FL with DH, we propose F3 (Sec.~\ref{sec::f3}) -- an FL framework for Fair FAC. With F3, we introduce novel aggregation heuristics: (i) \textsf{FairBest}, (ii) $\alpha$-\textsf{FairAvg}, and (iii) $\alpha$-\textsf{FairAccAvg} to improve the accuracy and fairness trade-off (Sec.~\ref{sec::heuristics}).
 
    \item Our heuristics outperform the state-of-the-art method, \texttt{FedAvg-DH}~\cite{fedavg}, ensuring 25\%-82\% improvement in fairness on popular face datasets (FairFace~\cite{fairface}, FFHQ~\cite{ffhq}, and UTK~\cite{utk} -- Table~\ref{table:biggg}).
    
    \item From our empirical results, we show that our heuristics achieve a practical trade-off between accuracy and fairness, with accuracy drop of 0.4\%-17\% compared to \texttt{FedAvg-DH}~\cite{fedavg} (Fig.~\ref{fig:scatter}).
    
\end{enumerate}

\section{Related Work}
\label{sec:rel_work}

\smallskip
\smallskip
\noindent\textbf{Balanced Face Datasets:} 
\cite{merler2019diversity} show that face datasets such as VGGFace2~\cite{vggface2}, MS-Celeb-1M~\cite{celeb1m}, MegaFace~\cite{megaface}, LFW~\cite{lfw} are imbalanced w.r.t. demographic groups such as gender and race. Research has also focused on constructing attribute balanced face datasets that include FairFace~\cite{fairface}, Balanced Faces in the Wild~\cite{bfw}, and Dive Face~\cite{sensnets}. However, the paradigm of training models on these data sets is typically centralized, unlike our FL setting, which assumes data heterogeneity.

\smallskip
\noindent\textbf{Fair Classification:}
Classical methods for fair classification~\cite{bechavod17,zafar17,calders10,dwork12,feldman15,kamiran09,kamishima11} ensure fairness in credit scoring, predicting criminal recidivism rates and individual incomes. A host of training methods incorporate fairness violation as part of optimisation~\cite{Jung_2021_CVPR,fnnc,lag-dual,fairconst}. The most popular approach among these is \emph{Lagrangian Multiplier Method} (LMM)~\cite{fairalm}. However, computing the fairness violation required in LMM is impossible in a heterogeneous data setting. A similar argument holds for~\cite{Jung_2021_CVPR}, where the authors consider knowledge distillation and use maximum-mean discrepancy based regularisation to ensure fairness. Researchers also consider semi-supervised \cite{jung2021} and unsupervised learning \cite{goyal2022vision} for fairness. However, these approaches implicitly assume data homogeneity across the sensitive attributes.

\smallskip
\noindent\textbf{Federated Learning (FL):} FL is employed in popular computer vision tasks such as image classification~\cite{fedv2,noniid1}, landmark classification~\cite{hsu2020federated}, and object detection~\cite{fedv1}. For most FL settings, the weighted-average ~(\texttt{FedAvg}~\cite{fedavg}) is used de facto. However, \texttt{FedAvg} is not designed to provide fairness. 
\cite{zeng2021improving} empirically show that models trained using \texttt{FedAvg} are more biased compared to a centrally trained model. While they propose an aggregation technique to improve fairness, it assumes data homogeneity which is at odds with our data heterogeneity setting. 

\smallskip
\noindent\textbf{FL with non-i.i.d. Data:} FL approaches  typically consider data heterogeneity for target labels  where each client has access only to a specific set of target labels, which is non-i.i.d. data.  \cite{noniid2} propose a reinforcement learning-based approach to improve accuracy in non-i.i.d. data. \cite{moon} use contrastive learning to ensure better performance. In addition, such aggregation heuristics for non-i.i.d. data~\cite{noniid2,noniid1} do \emph{not} consider fairness. Also, in all the above approaches, the heterogeneity is w.r.t. labels and not w.r.t. sensitive attributes.

{In summary, existing literature either considers fairness in a non-FL setting or only with data homogeneity in FL. In contrast, our framework, F3, holistically integrates fairness in FL setting with data heterogeneity. Also, unlike prior work, we consider data heterogeneity w.r.t. the sensitive attribute.}

In the next section, we formally define our FL setting and provide the fairness definitions.

\section{Preliminaries}

We consider Face Attribute Classification (FAC) task, where $\mathcal{X}$ is the universal set of face images, with binary labels $\mathcal{Y}=\{0, 1\}$ (e.g., male or female), and sensitive attribute $A \in \mathcal{A}$. Here, $A$ can be age, race, or gender. The sensitive attribute takes a finite set of values, $A = \{a_1, \ldots,a_s\}$. {E.g., the sensitive attribute \emph{age} can take values such as `young' or `adult'}. We next describe our FL setting for FAC.

\subsection{Federated Learning (FL) Setting\label{sec::prelimFL}}

In FL, the data is distributed across multiple parties referred to as clients. Let $\mathcal{C} = \{C_1,\ldots,C_m\}$ represent the set of clients; each $C_i$ owns a private and finite dataset $D_i \subset \mathcal{X} \times \mathcal{Y} \times A$ containing $n_i$ samples. Each $C_i$ trains its local model $h_{\theta_{i,t}}:\mathcal{X} \rightarrow \mathcal{Y}$ parameterized by $\theta_{i,t}$ at round $t$. 

At each round $t$, a random subset of clients $S_t\subseteq \mathcal{C}$ communicate their locally updated model parameters $\Theta_{t}=\{\theta_{i,t}~|~C_i\in S_t\}$ to the aggregator. The aggregator combines all communicated model parameters to obtain the global parameters at round $t$, $\phi_t$, using a heuristic choice function $\mu:\Theta_t\rightarrow \phi_t$. The aggregator then communicates the model parameters back to the clients. Then clients initialise their local model with these parameters and train further. This back and forth process is repeated multiple times till convergence. Fig.~\ref{fig:f3} provides an illustration of FL setting for F3 framework. 

We next present the fairness notions explored in this paper.
%
%
\begin{figure}[!t]
    \centering
    \includegraphics[width=0.9\linewidth,trim={10pt 100pt 10pt 10pt}]{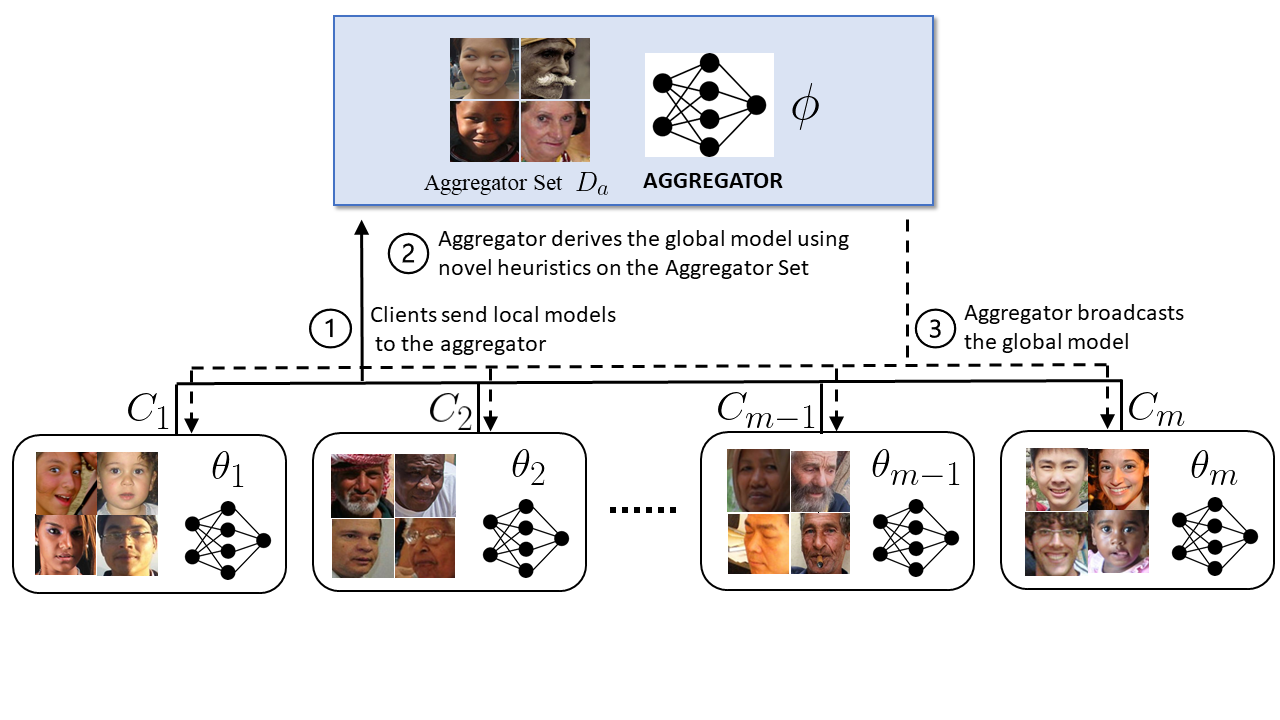}
    \caption{F3: Our FL Approach. The aggregator combines the client models $\theta_1,\ldots,\theta_m$ based on our proposed heuristics~(Sec.~\ref{sec::heuristics}) to obtain the global model $\phi$. Our heuristics depend on the fairness and accuracy scores computed on the aggregator set $D_a$. 
    }
    \label{fig:f3}
\end{figure}


\subsection{Fairness Notions}\label{FairNot}
The standard notions for fair classification depend on the error rates: False negative rate (FNR), and False positive rate (FPR). For a face attribute classifier $h$, given a face image $x$ with true label $y$ and sensitive attribute $a \in A$, we have $FNR = \Pr(h(x) \neq y | y = 1)$ and $FNR_a = \Pr(h(x) \neq y | A = a, y = 1), \forall a \in A$. Likewise, $FPR = \Pr(h(x) \neq y | y = 0)$, and $FPR_a = \Pr(h(x) \neq y | A = a, y = 0), \forall a \in A$. $FNR_a$ and $FPR_a$ are the error rates observed on the data samples belonging to a particular demographic group with sensitive attribute $a\in A$. E.g., consider a FAC task for `gender' classification with `age' as the sensitive attribute. The attribute comprises \{`young', `adult'\} as the demographic groups. Now, consider the following group-fairness notions.

\smallskip
\noindent\textbf{Equality of Opportunity (EOpp) \cite{chouldechova2017fair}:} A classifier $h$ satisfies EOpp for a distribution over $(\mathcal{X}, \mathcal{Y}, A)$ if: $FNR_a = FNR, \forall a$. We denote the violation in EOpp as $\Delta_{EOpp} = \max(\{ FNR_a - FNR | \forall a \in A\})$. That is, $\Delta_{EOpp}$ is the maximum disparity in $FNR$ across the demographic groups. Intuitively, EOpp ensures that the probability of predicting a `male' face as `female' is the same across age groups.

\smallskip
\noindent\textbf{Equalized Odds (EO)~\cite{hardt2016equality}:} A classifier $h$ satisfies EO over $(\mathcal{X}, \mathcal{Y}, A)$ if: $FNR_a = FNR$ and $FPR_a = FPR ~\forall a$. $\Delta_{EO} = \max(\max(\{FPR_a - FPR | \forall a \in A\}), \max(\{FNR_a - FNR | \forall a \in A\}))$ denotes violation in EO. EO states that the probability that the model predicts a `male' face to be `female' is independent of age.

\smallskip
\noindent\textbf{Accuracy Parity (AP)~\cite{AP}:} A classifier $h$ satisfies AP for a distribution over $(\mathcal{X}, \mathcal{Y}, A)$ if: $FPR + FNR = FPR_a + FNR_a, \forall a$. $\Delta_{AP}= \max(\{FPR_a - FPR | \forall a \in A\}) + \max(\{FNR_a - FNR | \forall a \in A\})$ denotes violation in AP. AP states that the overall classification error must be equal across the age groups.


To incorporate these fairness notions in FAC, the standard technique is to train a model that maximises accuracy while minimising the violation in these fairness notions. {Towards this, the state-of-the-art approach, Lagrangian Multiplier Method (LMM)~\cite{fairalm}, adopts a loss function that simultaneously incorporates cross-entropy loss $l_{CE}$ and the violation in fairness constraint $(\Delta_{EOpp}, \Delta_{EO}, \Delta_{AP})$, weighted by the lagrangian multiplier $\lambda \in \mathbb{R}^{+}$.} 

Formally, in LMM, the loss $L_{LMM}(h(X),Y, A)$ for a classifier $h$, for $k \in \{EOpp, EO, AP\}$ and $(X, Y) \subseteq \mathcal{X} \times \mathcal{Y}$, is as follows.

\begin{align}\label{eq:lag}
    L_{LMM}(\cdot) =\mathbb{E}_{(x,y)\sim (X, Y)}[l_{CE}(h(x), y)]+\lambda\Delta_{k}.
\end{align}

{LMM requires each client to possess data belonging to every demographic group. Later in Proposition~\ref{prop::inf}, we show that LMM will \emph{not} work in FL with DH, necessitating the need for newer approaches. Motivated by this, we propose F3, a framework for ensuring fairness in FL with DH.}

\section{Methodology}

We begin by motivating the problem and formally show that existing FL approaches for ensuring fairness cannot be applied in our setting. Subsequently, we describe our framework F3, including the proposed heuristics.

\subsection{Motivation}

\label{sec:fac_DH}
Existing approaches aggregate locally trained fair models to obtain a globally fair model~\cite{fairalm}. Typically, each FL client trains a local classifier imposing fairness constraints. These methods assume data homogeneity w.r.t. sensitive attributes across clients. However, this assumption is practically restrictive.

In reality, each client might only possess samples from an individual demographic group. E.g., samples belonging to `young' age group when age is sensitive attribute. We refer to this scenario as Federated Learning with Data Heterogeneity (FL with DH). 
However, as we show next with Proposition~\ref{prop::inf}, the issue with DH is that the error rates for the demographic groups not present in a particular client's data are not defined. As a result, the client's fairness violation component (Sec.~\ref{FairNot}) cannot be computed.

\begin{proposition}\label{prop::inf}
In Lagrangian Multiplier Method (LMM),  $L_{LMM}$ (Eq.~\ref{eq:lag}) -- the loss function for training the classifier is not defined in FL with DH.
\end{proposition}
 \begin{proof}

 W.l.o.g., let the $\Delta_{k}$ component of $L_{LMM}$ in Eq.~\ref{eq:lag} be $\Delta_{EOpp}$. For $C_i$ with local data $D_i$, $\Delta_{EOpp} = \left| \frac{\sum\limits_{j=1}^{|D_i|} (1 - p_j) y_j a_j}{\sum_{j=1}^{|D_i|} a_j} - \frac{\sum\limits_{j=1}^{|D_i|}(1 -  p_j) y_j (1 - a_j)}{\sum_{j=1}^{|D_i|} (1-a_j)}\right|$, where $p_j = h_i(x_j) \ \forall j \in \{1,\ldots, |D_i|\}$ and binary sensitive attribute $a_j \in \{0,1\}$~\cite{fnnc}. Under DH, every sample in $D_i$ will belong to only a single demographic group (e.g., $\leq$ 30 age group), i.e., either $a_j = 1, \forall j$ or $a_j=0, \forall j$. In both cases,  $\Delta_{EOpp}$ tends to $\infty$ as either $\sum_{j=1}^{|D_i|} (1-a_j) = 0$ or $\sum_{j=1}^{|D_i|} a_j = 0$;  consequently, $L_{LMM} \rightarrow \infty$. \qed
 \end{proof}

Proposition~\ref{prop::inf} holds for any fairness violation function (including $\Delta_{EO}, \Delta_{AP}$) that requires samples belonging to all the demographic groups. E.g., the loss functions defined in \cite{agarwal18,zafar17,zhang18}. Thus, we cannot use these loss functions to train for fairness in FL with DH. Further, training only for accuracy compromises fairness~\cite{chouldechova17}, implying that standard approaches such as \texttt{FedAvg} cannot be adopted. As a result, we propose F3, a novel FL framework for fair FAC that employs different aggregation heuristics for achieving fairness.

\begin{figure*}[!t]
\begin{tabular}{c|c}
    \begin{subfigure}{0.5\textwidth}
    \centering
    \includegraphics[width=\columnwidth]{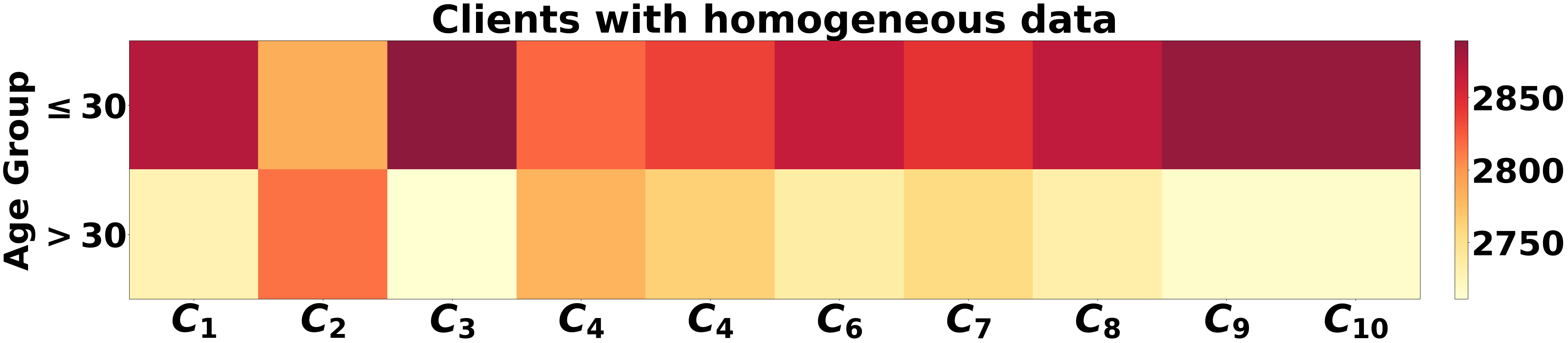}
    \end{subfigure} &
    \begin{subfigure}{0.5\textwidth}
    \centering
    \includegraphics[width=\columnwidth]{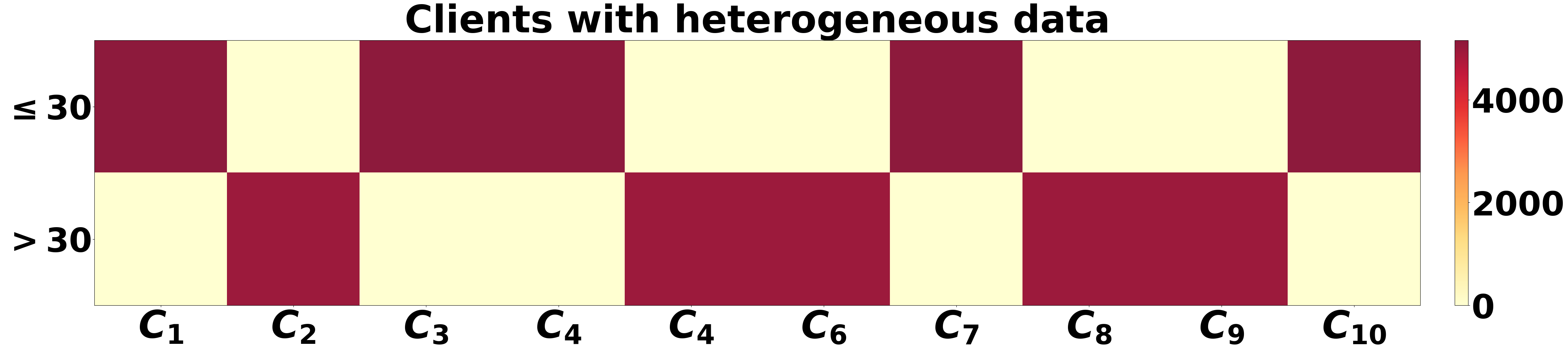}
    \end{subfigure}
    \end{tabular}
    \caption{An illustration of data heterogeneity across clients  with age as the sensitive attribute (FFHQ data set~\cite{ffhq}). In the \emph{heterogeneous} setting (right plot), each client has data belonging to only a specific age group ($\leq$ 30 or $>$ 30). In contrast, the \emph{homogeneous} setting (left plot) shows that each client has data almost equally distributed across both the age groups (See Section~\ref{sec:fac_DH}).}
    \label{fig:DH}
\end{figure*}

\begin{algorithm}[!t] 
\small
\caption{\label{algo::F3} F3 Framework}
\begin{algorithmic}[1]

\Require{(1) Each client $C_k \in \mathcal{C}$ has its private dataset $D_k$ s.t. $\theta_{k,0}\leftarrow \theta_{0}$. (2) Hyperparameters: maximum number of communication rounds $T$, number of local epochs $E$, learning rate $\eta$, accuracy tolerance $a$, threshold round $\tau$. (3) A heuristic choice function $\mu_\varsigma( \Theta_t)$ s.t. $\varsigma=\{\textsf{FairBest},$ $\alpha\mbox{-}\textsf{FairAvg},\alpha\mbox{-}\textsf{FairAccAvg}\}$ }.

\Ensure{Aggregated model $\phi$}
 \Procedure{LocalTraining$(k, \phi_{t})$}{}
\State $\theta_{k,t}\leftarrow \phi_t$
        \For{each local epoch $i = 1,2,\ldots,E$}
            \State{$\theta_{k,t} \leftarrow \theta_{k,t} - \eta\cdot\nabla_{\theta_{k,t}}L_k(h_{\theta_{{k,t}}}(\cdot),D_k)$}
        \EndFor
        \State{\Return $\theta_{k,t}$ to the aggregator}
  \EndProcedure
%
%
  \Procedure{Aggregation$(\mu_\varsigma(\cdot))$}{}
    \For {\text{each round $t = 1,2,\ldots,T-1$}}
        \For{each client $k \in S_t$ (in parallel)}
            \State{$\theta_{k,{t}} \leftarrow \textsc{LocalTraining}(k,\phi_{t})$}
        \EndFor
        \State $\phi_{t+1} \leftarrow \mu_\varsigma( \Theta_t)$ where $\Theta_{t}=\{\theta_{k,t}~|~k\in S_t\}$

        \If{$t>\tau$}   
        \State $\Delta_{Acc} = |Acc(h(\phi_t)) -$ \Statex $\qquad\qquad\qquad\qquad\max\left(Acc(h(\phi_{t-1})),\ldots,Acc(h(\phi_{t-\tau}))\right)|$
        
        \If{$\Delta_{Acc}<a$}
        \State \Return $\phi_{best}$  
        \Comment{\footnotesize{\textcolor{blue}{Training stops; As $\Delta_{Acc}$, across ``$\tau$" communication rounds is less than tolerance ``$a$"}}}
        
        \ElsIf{$\Delta_{Acc}>a$}
        \State $\phi_{best} \leftarrow \phi_{t+1}$
        
        \ElsIf{$\Delta_{Acc}>0$}
        \State $\phi_{best} \leftarrow \textsf{FairCheck}(\phi_{best},\phi_{t+1})$  
        \Comment{\footnotesize{\textcolor{blue}{outputs the model parameters that produce lesser fairness violation}}}
        \EndIf
        
        \Else
        \State $\phi_{best} = \phi_{t+1}$
        \EndIf
    \EndFor
   \State \Return $\phi_{best}$
 \EndProcedure
\end{algorithmic}
\end{algorithm}
\normalsize
%
%

\subsection{F3 Framework\label{sec::f3}}

The ingenuity of our framework, F3, is to adopt different aggregation heuristics that prioritise the local client models, which perform desirably in terms of fairness. We next briefly summarise F3 as follows.

\smallskip
\begin{enumerate}[noitemsep, leftmargin=*]
    \item \textit{Local Training}. Each $C_i$ trains its model $h_{\theta_i}$ only for maximising accuracy (i.e., minimising cross-entropy loss $l_{CE}$), $L_i(h_{\theta_i}, D_i) = \mathbb{E}_{(x,y) \sim D_i}[l_{CE}(h_{\theta_i}(x), y)]$. A subset of clients $S_t$, at each round $t$, communicate their model parameters to the aggregator. 
    \item \textit{Model Aggregation}. For aggregation, we propose novel heuristic choice functions ($\mu$) to control the accuracy and fairness trade-off~(Sec.~\ref{sec::heuristics}). We consider an aggregator set $D_a$ comprising samples belonging to each demographic group to execute these heuristics~\cite{val}. Note that $D_a$ has a limited number of samples to be used directly for training. Aggregator derives a global model employing our heuristics~(Sec.~\ref{sec::heuristics}).
    \item \textit{Model Communication.} The aggregator then communicates the global model parameters to each client. The clients initialise their models with these parameters and further train on it to maximise their accuracy.
\end{enumerate}

Fig.~\ref{fig:f3} depicts F3 framework, and Algorithm~\ref{algo::F3} provides a procedural outline of F3. Next, we introduce various heuristic choice functions used in Model Aggregation step.

\subsection{Heuristics for Fair FL\label{sec::heuristics}}
{Recall that in FL with DH, firstly, we cannot train the local client models for fairness. Thus, standard heuristics for ensuring fairness do not work in our setting (Proposition~\ref{prop::inf}). Next, as \texttt{FedAvg} aggregates a random (sub)set of models at each round, it fails to ensure fairness as the models for aggregation may potentially be biased. In turn, they may amplify the unfairness in the global model.}

Contrary to this, our key idea is to achieve fairness in FL with DH by deliberately selecting the subset of local models for aggregation that perform desirably w.r.t. to fairness and accuracy (Step 2 in Fig. \ref{fig:f3}). The aggregator quantifies the performance of local client models based on their empirical fairness violation and accuracy computed on the aggregator set $D_a$. Specifically, let $\Delta_{loss}(h_{i}(\theta_{i,t}))$ denote any fairness violation for $C_i$'s model on $D_a$ at any round $t$. Also, $Acc(h_{i}(\theta_{i,t}))$ is the accuracy of $C_i$'s model over $D_a$ at $t$. Based on this, we introduce the following heuristics that aim to strike a practical balance between fairness and accuracy.

\begin{itemize}[leftmargin=*]

\item[$\bullet$] {\textsf{FairBest}}: In this, aggregator selects a specific model from the set of local models, which provide the least fairness violation on the aggregator set $D_a$. That is, the global aggregation parameter at a round $t$, are,
$$
{\mu_{\textsf{FairBest}}(\Theta_t)  \triangleq \phi_{t} =\theta_{i^*,t}~\mbox{~s.t.~} i^* = \underset{i}{\argmin}\left\{\Delta_{loss}(h_{i}(\theta_{i,t}))\right\}}
$$

\smallskip
\item[$\bullet$] {$\alpha$-\textsf{FairAvg}}: This heuristic generalizes \textsf{FairBest} by selecting the top $\alpha$\% of local models and then take their weighted average. More formally, consider the set $F_t$ which comprises the top-$\alpha$\% of clients in \textit{increasing} order of $\Delta_{loss}(h_{i}(\theta_{i,t}))$. Now,

\begin{equation}\label{eqn::fairavgNEW}
 \mu_{\alpha\mbox{-}\textsf{FairAvg}}(\Theta_t)  \triangleq  \phi_{t} = \sum_{i\in F_t}\frac{n_i}{\sum_{j\in F_t} n_j}\theta_{i,t}
\end{equation}

\smallskip
\item[$\bullet$] {$\alpha$-\textsf{FairAccAvg}}:  Aggregator selects the top-$\alpha$\% of local model parameters that give the best ratio of accuracy with fairness violation on $D_a$ and take their weighted average. Consider the set $F_t$ which comprises the top-$\alpha$\% of clients in \textit{decreasing} order of the ratio $\frac{Acc(h_{i}(\theta_{i,t}))}{\Delta_{loss}(h_{i}(\theta_{i,t}))}$. Again,

\begin{align}\label{eqn::fairaccNEW}
 \mu_{\alpha\mbox{-}\textsf{FairAccAvg}}(\Theta_t)  \triangleq  \phi_{t} = \sum_{i\in F_t}\frac{n_i}{\sum_{j\in F_t} n_j}\theta_{i,t}
\end{align}

\end{itemize}

{Observe that as $\alpha$ increases, more and more local models are considered for aggregation akin to \texttt{FedAvg} with heterogeneous data. That is, with an increase in $\alpha$, Eq.~\ref{eqn::fairavgNEW} and Eq.~\ref{eqn::fairaccNEW} tend to $\sum_{i\in \mathcal{C}} \frac{n_i}{\sum_j n_j} \theta_{i,t}$ (\texttt{FedAvg} aggregation). That is, $\alpha$-\textsf{FairAvg} and $\alpha$-\textsf{FairAccAvg} tend to mimic \texttt{FedAvg} (see Fig.~\ref{fig:Alpha-plot}).}

Given these heuristics, we provide F3's performance in terms of accuracy and fairness violation on three real-world face datasets in the next section.

\section{Experiments\label{sec::exps}}

%
\begin{table}[t]
    \centering
    \adjustbox{max width=\textwidth}{
    \begin{tabular}{|c|c|c|c|c|c|c|}
    \hline
    \textbf{Dataset}     &  \textbf{Input Size} & \textbf{\# of faces } & \multicolumn{2}{c|}{\textbf{Sensitive Attribute (Age)}} & \multicolumn{2}{c|}{\textbf{Face Attribute Classification (FAC)}} \\

\hline
\multirow{2}{*}{FairFace~\cite{fairface}  }    & \multirow{2}{*}{$128\times 128\times 3$}  & \multirow{2}{*}{$\approx$98K} & \multicolumn{2}{c|}{\hfill $\leq 30$ years\hfill\vline\hfill $> 30$ years\hfill\null}  & \multicolumn{2}{c|}{Gender} \\
\cline{6-7}
    &  &  & \multicolumn{2}{c|}{\hfill ($\approx$53K)\hfill\vline\hfill   ($\approx$45K)\hfill\null} & \multicolumn{2}{c|}{\hfill Male ($\approx$51K)\hfill\vline\hfill Female ($\approx$45K)\hfill\null} \\
\hline
\multirow{2}{*}{FFHQ~\cite{ffhq} }    & \multirow{2}{*}{$128\times128\times3$}  & \multirow{2}{*}{$\approx$70K} & \multicolumn{2}{c|}{\hfill $\leq 30$ years\hfill\vline\hfill $> 30$ years\hfill\null}  & \multicolumn{2}{c|}{Gender} \\
\cline{6-7}
    &  &  & \multicolumn{2}{c|}{\hfill ($\approx$35K)\hfill\vline\hfill   ($\approx$34K)\hfill\null}    & \multicolumn{2}{c|}{\hfill Male ($\approx$32K)\hfill\vline\hfill Female ($\approx$37K)\hfill\null} \\

\hline
\multirow{2}{*}{UTK~\cite{utk}  }    & \multirow{2}{*}{$48\times 48\times 1$}  & \multirow{2}{*}{$\approx$23K} & \multicolumn{2}{c|}{\hfill $\leq 30$ years\hfill\vline\hfill $> 30$ years\hfill\null} & \multicolumn{2}{c|}{Gender} \\
\cline{6-7}
    &  &  & \multicolumn{2}{c|}{\hfill ($\approx$12K)\hfill\vline\hfill   ($\approx$11K)\hfill\null}     & \multicolumn{2}{c|}{\hfill Male ($\approx$6K)\hfill\vline\hfill Female ($\approx$4K)\hfill\null} \\
\hline
    \end{tabular}}
    \caption{Summary Statistics of Face Datasets}
    \label{tab:datasets}
\end{table}

We conduct our experiments on the following face datasets: FairFace \cite{fairface}, FFHQ \cite{ffhq}, and UTK \cite{utk}. In this section, we first define our baseline \texttt{FedAvg-DH} for an appropriate comparison. Then, we provide our network architecture, training details and the FL setup. Finally, we present our results and the key takeaways.

\smallskip
\noindent\textbf{Baseline:} To validate our proposed heuristics, we compare their performance with the baseline \texttt{FedAvg-DH} in terms of accuracy and violation of fairness notions mentioned in Sec.~\ref{FairNot}. \texttt{FedAvg-DH} is simply \texttt{FedAvg} for our FL setting with Data Heterogeneity.

\smallskip
\noindent\textbf{FL Setup:} In our FL setting, for the baseline and our heuristics, we consider 50 clients, i.e., $\mathcal{C}=\{C_1,\ldots,C_{50}\}$. We randomly distribute the training data such that each client has data samples of only a particular demographic group to ensure data heterogeneity. Each client's model is locally trained on its private data. The global model aggregation is performed periodically till convergence. The training details specific to each dataset follow next. At each aggregation round $t$, we let $S_t=\mathcal{C}$.

\smallskip
\noindent\textbf{Training Details:}
We focus on three popular face datasets: FairFace~\cite{fairface}, FFHQ~\cite{ffhq}, and UTK~\cite{utk}. 
For each of these, we consider `age' as the sensitive attribute and `gender' as the predicting label. Further, we divide the samples into two age groups, $\leq 30$ and $> 30$ years. We distribute the data among the clients such that 50\% of the clients have access to data samples belonging to the age group $\leq 30$ and others have access to samples belonging to age group $> 30$. For all three datasets, each client's local data comprises $\approx$ 1K training samples. 

For FairFace~\cite{fairface} and FFHQ~\cite{ffhq}, we use a batch size of 256 and train the models for $T=50$ communication rounds with clients training their models locally for $E=4$ epochs (per round). We use learning rates of $\eta= 0.05$ and $\eta=0.01$ for FairFace and FFHQ, respectively. For UTK~\cite{utk}, we train using batch size 64 for $T=80$ and $E=2$ rounds. We also set the learning rate as $\eta = 0.01$. For all three datasets, we set the accuracy tolerance at $a=1\%$ and the threshold round at $\tau=20$.

\smallskip
\noindent\textbf{Model:} We adopt PyTorch's implementation of the standard ResNet-18 architecture for the base model \cite{fairalm}. We use SGD optimization in our training process. We run our experiments on 8 NVIDIA GeForce GTX 1080 with $10$ GB RAM.

\smallskip
\noindent\textit{Stopping Criteria.} Fig.~\ref{fig:AvgEpochs} depicts the change in accuracy and $\Delta_{EO}$ values over communication rounds for 10-\textsf{FairAvg} and 10-\textsf{FairAccAvg}. We remark that fairness guarantees often come at the cost of accuracy~\cite{lmm-acc}. However, our heuristics ensure that the fairness violation does not substantially increase as the accuracy improves. Our fair model selection and stopping criteria (Line 13, Lines 16-17 in Algorithm \ref{algo::F3}) ensure a practical accuracy and fairness trade-off. 


\begin{figure*}[!t]
    \centering
    \includegraphics[width=0.9\linewidth]{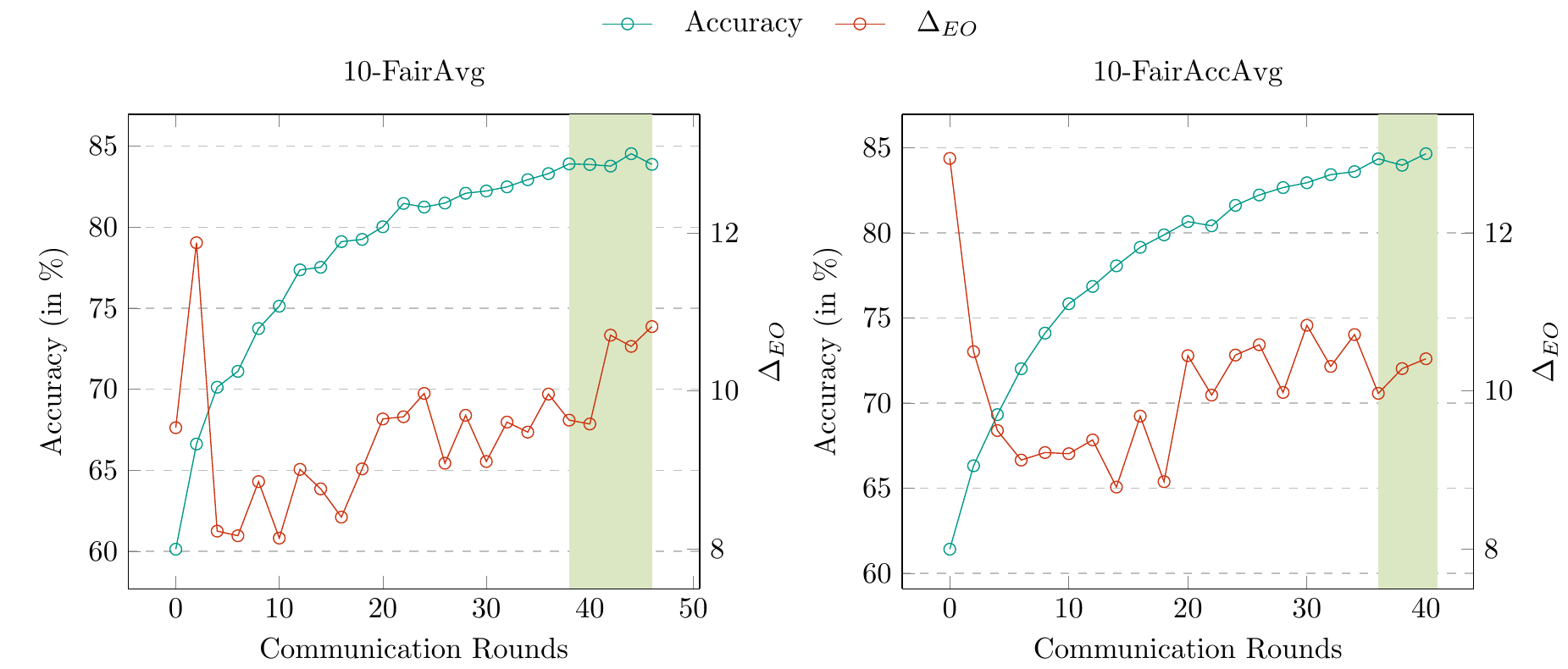}
    \caption{Accuracy and Fairness values over communication rounds for \textsf{10-FairAvg} and \textsf{10-FairAccAvg}. Observe that as training progresses, the fairness violation $(\Delta_{EO})$ does not significantly worsen while accuracy gradually increases.  The shaded region depicts the rounds where training is stopped -- i.e., the change in accuracy, $\Delta_{\text{Acc}}$,  (Algorithm~\ref{algo::F3}, Line 15) is below 1\%.}
    \label{fig:AvgEpochs}
\end{figure*}

\subsection{Results}\label{sec:results}
We now compare accuracy and fairness violations across different heuristics for each of the three datasets. We run every experiment 5 times and report the average and the standard deviation of the measures. For each instance, we randomly generate a aggregator set $D_a$ with the total number of samples between 10\%-20\% of the overall dataset size. Table.~\ref{table:biggg} provides the accuracy and fairness violation values of each of our proposed heuristics in comparison to \texttt{FedAvg-DH}. The maximum Coefficient of Variation (CoV)~\cite{cov} observed across our experiments is 0.96 and less than 0.2 for almost 70\% of the time. This indicates the stability of our approach and the results presented in this work. While we report our results for $\alpha$-\textsf{FairAvg} and $\alpha$-\textsf{FairAccAvg} with $\alpha = 10$ (10\% of the total local models), Fig.~\ref{fig:Alpha-plot} illustrates the accuracy and fairness trade-off on FairFace w.r.t. $\alpha$.

\begin{figure*}[!t]
    \centering
    \includegraphics[width=0.8\linewidth]{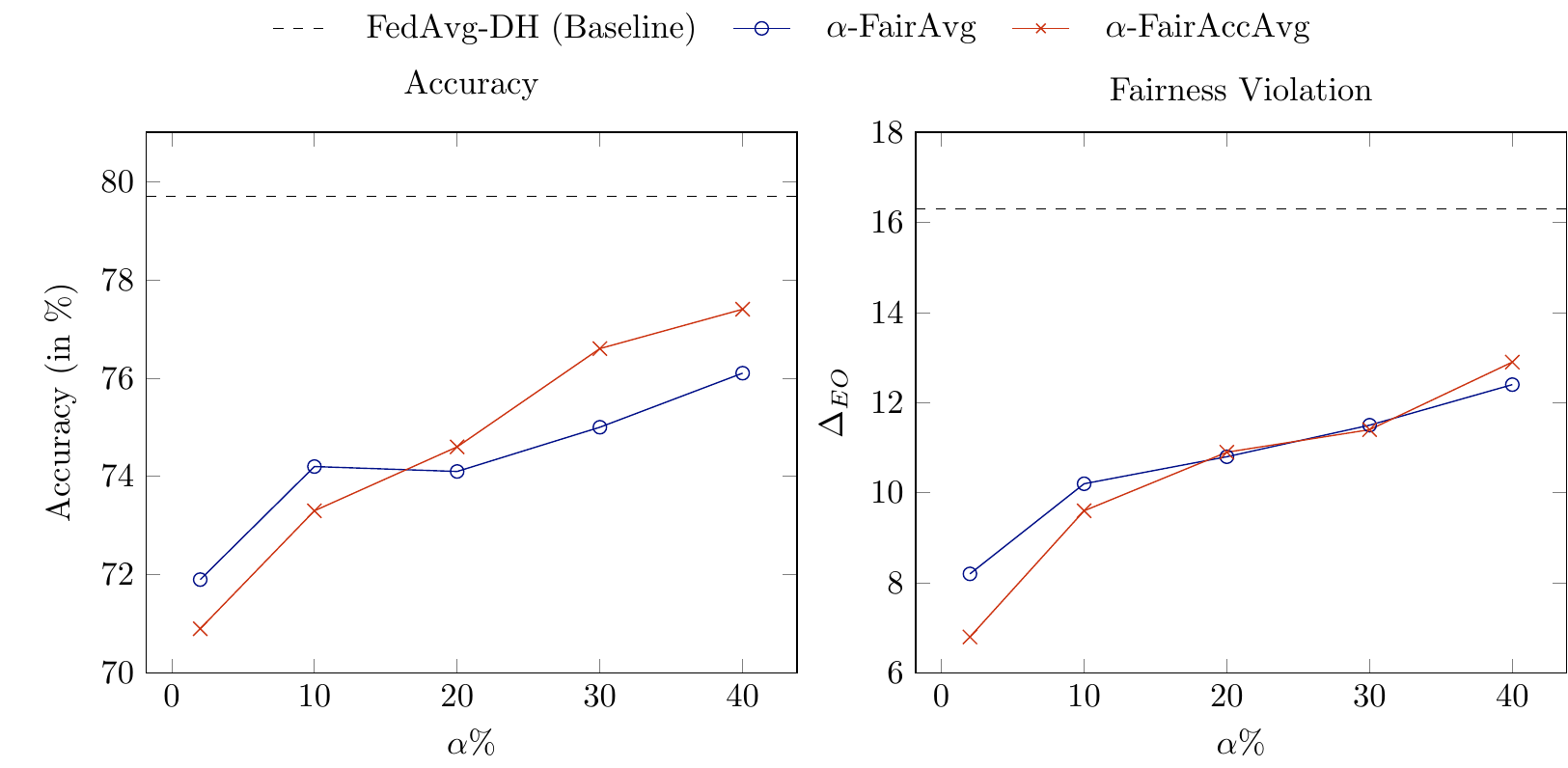}
    \caption{Accuracy and Fairness Violation with varying $\alpha$. We compare accuracy (left plot) and $\Delta_{EO}$ (right plot) of $\alpha$-\textsf{FairAvg} and $\alpha$-\textsf{FairAccAvg} on FairFace for different $\alpha$ values. Note that for smaller $\alpha$, we get improved fairness (50\% reduction in $\Delta_{EO}$ for $\alpha=10$) with marginal reduction in accuracy ($< 8\%$ ).}
    \label{fig:Alpha-plot}
\end{figure*}

\begin{table*}[!t]
    \centering
    \adjustbox{max width=\textwidth}{%
\begin{tabular}{|c|c|c|c|c|c|c|c|c|c|c|c|c|}
  \hline
  & \multicolumn{4}{c|}{{FairFace}~\cite{fairface}} & \multicolumn{4}{c|}{{FFHQ}~\cite{ffhq}} & \multicolumn{4}{c|}{{UTK}~\cite{utk}} \\
  \hline
 \multirow{2}{*}{{Heuristic}} &  \multirow{2}{*}{{Accuracy  ($\uparrow$)}} & 
 \multicolumn{3}{c|}{Reduction in Fairness Violation ($\downarrow$)} & \multirow{2}{*}{{Accuracy  ($\uparrow$)}} & 
 \multicolumn{3}{c|}{Reduction in Fairness Violation ($\downarrow$)} & \multirow{2}{*}{{Accuracy  ($\uparrow$)}} & 
 \multicolumn{3}{c|}{Reduction in Fairness Violation ($\downarrow$)}\\
 \cline{7-9} \cline{3-5} \cline{11-13}
   &  & \textbf{$\Delta_{EOpp}$} & \textbf{$\Delta_{EO}$}& \textbf{$\Delta_{AP}$} & & \textbf{$\Delta_{EOpp}$} & \textbf{$\Delta_{EO}$}& \textbf{$\Delta_{AP}$} & & \textbf{$\Delta_{EOpp}$} & \textbf{$\Delta_{EO}$}& \textbf{$\Delta_{AP}$}\\
 \hline
 \texttt{FedAvg-DH} & 79.7\% $\pm$ 0.8  & {$1.0 \pm 1.4$}  & {$16.3 \pm 3.3$}  & {$17.3 \pm 2.5$} & 90.8\% $\pm$ 0.5  & {2.4 $\pm$ 1.3}  & {$10.5\pm 1.3$} & {$13.0 \pm 0.5$} & 94.1\% $\pm$ 0.5  & {$1.5 \pm 1.2$}  & {$11.9 \pm 1.6$}  & {$13.5 \pm 2.7$}\\
 
 \textsl{FairBest} & 72.0\% $\pm$ 0.8  & \cellcolor{green!15}{$0.6 \pm 0.6$}  & \cellcolor{green!15}{$\;8.2\; \pm 1.1$}  & \cellcolor{green!15}{\;8.8\; $\pm$ 1.1} & 74.7\% $\pm$ 2.0  & \cellcolor{green!15}{1.8 $\pm$ 0.4}  & \cellcolor{green!15}{\;1.9\; $\pm$ 0.3} & \cellcolor{green!15}{\;3.2\; $\pm$ 0.8}  & 83.9\%  $\pm$ 0.5 & {$1.9 \pm 4.6$}  & {\;6.0\; $\pm$ 4.3}  & {10.0 $\pm$ 6.9}\\
 

 $10$-\textsl{FairAvg} & \cellcolor{magenta!25}{{74.2\% $\pm$ 1.1}}  & {3.3 $\pm$ 1.2}  & {10.2 $\pm$ 1.6} & {13.5 $\pm$ 2.5} & 82.2\% $\pm$ 0.5  & {2.7 $\pm$ 0.3}  & {\;2.7\; $\pm$ 0.3} & {\;4.3\; $\pm$ 0.6}  & 93.5\% $\pm$ 2.9  & \cellcolor{green!15}{0.9 $\pm$ 4.3} & {\;6.1\; $\pm$ 2.9}  & {\;7.8\; $\pm$ 4.1}\\

 $10$-{{\textsl{FairAccAvg}}} & 73.3\% $\pm$ 1.2  & {0.8 $\pm$ 0.7}  & {\;9.6\; $\pm$ 0.9}  & {10.4 $\pm$ 1.6} & \cellcolor{magenta!25}{{{82.7\% $\pm$ 0.5}}}  & \cellcolor{green!15}{1.8 $\pm$ 1.0}  & {\;2.0\; $\pm$ 0.5} & {\;3.9\; $\pm$ 1.1} & \cellcolor{magenta!25}{{{93.7\% $\pm$ 1.1}}}  & {3.1 $\pm$ 3.1}  & \cellcolor{green!15}{\;5.3\; $\pm$ 1.0} & \cellcolor{green!15}{\;6.8\; $\pm$ 2.1}\\

\hline
\end{tabular}}
\caption{Accuracy and Fairness Violations, $\Delta_k$, $k = \{$EOpp, EO, AP$\}$ on FairFace, FFHQ, and UTK of our heuristics compared to the baseline \texttt{FedAvg-DH}. Lower $\Delta_k$ ensures better fairness. The numbers in green highlight represent least value of $\Delta_{k}$ obtained by one of our heuristic, for each dataset. The values are significantly lesser than \texttt{FedAvg-DH}.  $10$-{{\textsl{FairAccAvg}}} and $10$-{{\textsl{FairAvg}}} (highlighted in magenta) provide the highest accuracy out of our proposed heuristics. }
\label{table:biggg}
\end{table*}

\smallskip
\noindent\textbf{Fairness Improvements:}   {We now compare $\Delta_k,\; k \in$ \{EOpp, EO, AP\} across our novel heuristics with \texttt{FedAvg-DH} as our baseline. From Table.~\ref{table:biggg}, observe that, in general, each of our heuristics outperform \texttt{FedAvg-DH} in terms of fairness. For FairFace, \textsf{FairBest} obtains fairness improvement of 40\% ($\Delta_{EOpp}$), 50\% ($\Delta_{EO}$), 50\% ($\Delta_{AP}$). For FFHQ, \textsf{FairBest} provides an improvement of 82\% for $\Delta_{EO}$. Further, all our proposed heuristics provide a minimum of 60\% improvement in fairness for FFHQ. For UTK, 10-\textsf{FairAccAvg} obtains a improved fairness of 50\% in $\Delta_{EO}$ and 55\% in $\Delta_{AP}$.}

\smallskip
\noindent\textit{Accuracy and Fairness Trade-off:}  Fig.~\ref{fig:scatter} illustrates the accuracy and fairness trade-offs of our heuristics with the baseline \texttt{FedAvg-DH}. The heuristics in the bottom-left corner of the plot assure the least fairness violation while maintaining high accuracy. As expected, \texttt{FedAvg-DH} (red-cross, Fig.~\ref{fig:scatter}) mostly guarantees the best accuracy but suffers from significant fairness violation. However, most of our heuristics provide lower fairness violations, for a marginal reduce in accuracy. We now elaborate on our results for each dataset, from Table~\ref{table:biggg}.

\noindent\underline{FairFace~\cite{fairface}:} {With our heuristic 10-\textsf{FairAccAvg}, we observe fairness improvements up to 42\% with an accuracy drop of only 7\% compared to \texttt{FedAvg-DH}. \textsf{FairBest} provides the least violation of fairness and 10-\textsf{FairAvg} provides the highest accuracy with FairFace.}

\noindent\underline{FFHQ~\cite{ffhq}:} {On FFHQ, we observe a fairness improvement up to 82\%. While 10-\textsf{FairAccAvg} guarantees the best accuracy, \textsf{FairBest} obtains 25\%, 82\% and 75\% reduction in $\Delta_{EOpp}, \Delta_{EO}$ and $\Delta_{AP}$ respectively. \textsf{10-FairAccAvg} provides a desirable trade-off between accuracy and fairness compared to others.}

\noindent\underline{UTK~\cite{utk}:} {For UTK, both 10-\textsf{FairAvg} and 10-\textsf{FairAccAvg} outperform \texttt{FedAvg-DH}. 10-\textsf{FairAvg} provides improved fairness by 40\% ($\Delta_{EOpp}$), 48\% ($\Delta_{EO}$) and 42\% ($\Delta_{AP}$) for an accuracy drop of only 0.6\%. Similarly, 10-\textsf{FairAvg} provides improved fairness by 55\% in $\Delta_{EO}$ and 49\% in $\Delta_{AP}$ for an accuracy drop of 0.4\%.}

\begin{figure*}[!t]
    \centering
    \includegraphics[width=\linewidth]{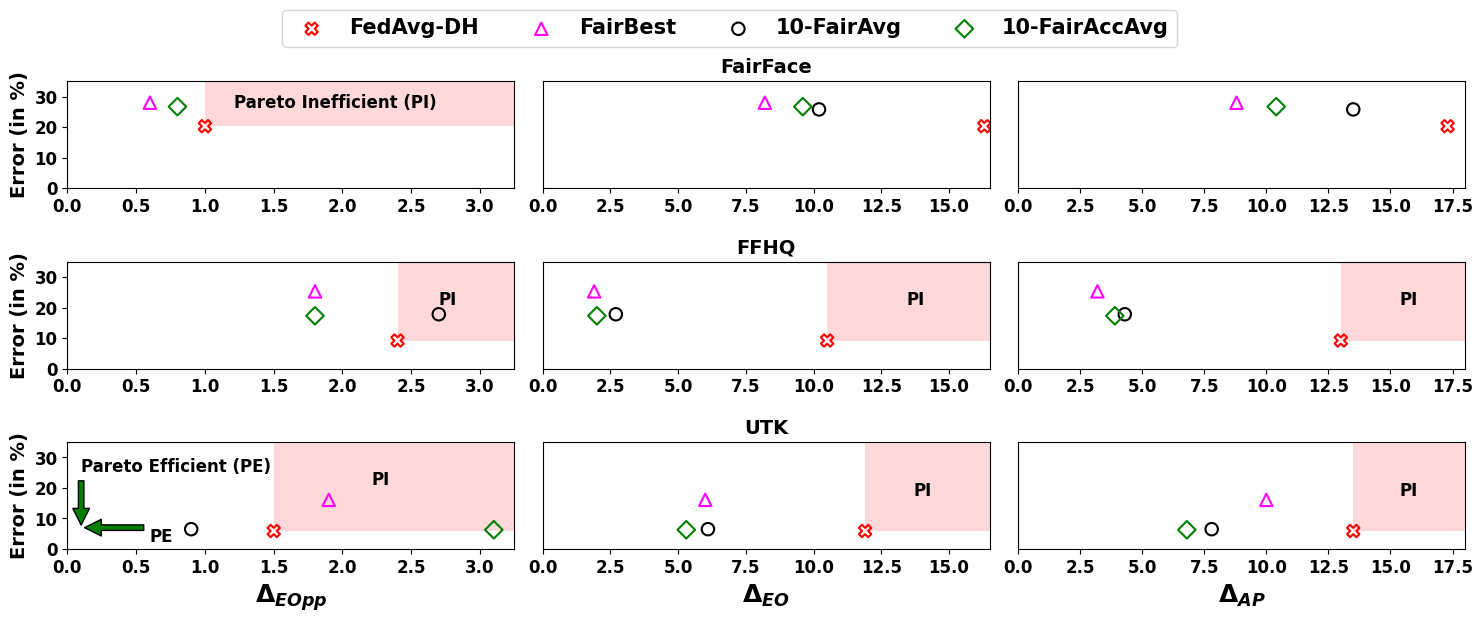}
    \caption{Accuracy and fairness trade-off. Note that the optimum point is bottom left, i.e., low \%-Error and low fairness violation ($\Delta_k$). \texttt{FedAvg-DH} (red cross marker) appears at the bottom right exhibiting low error (higher accuracy) at the cost of fairness. Our heuristics consistently provide a reduction in $\Delta_k$ -- for only a marginal increase in \% Error. The highlighted ``Pareto Inefficient" region is the area which is Pareto dominated by \texttt{FedAvg-DH}. Observe that, our proposed heuristics mostly lie outside the Pareto inefficient region.
    }
    \label{fig:scatter}
\end{figure*}

To rank performance w.r.t. accuracy-fairness trade-off, we use \emph{Mahalanobis} distance~\cite{bishop2006} between our heuristics' performance (Error and Fairness violation) from the origin (see Fig.~\ref{fig:scatter}).
The heuristic that achieves the least distance achieves a better trade-off. We observe that 10-\textsf{FairAccAvg} performs better than the other two heuristics for EOpp on FairFace, EO and AP on UTK, and for all three fairness notions on FFHQ, indicating its desirability. We provide the complete results in the Appendix.
In summary, we see that our heuristics perform remarkably well in terms of fairness while maintaining competitive accuracy compared to  \texttt{FedAvg-DH}. 

\smallskip
\noindent\textbf{Oracle to Achieve Data Homogeneity:} Our approach is designed to operate under Data Heterogeneity. We now demonstrate that even if we impose the relatively more restrictive condition of data homogeneity\footnote{Each client's local data comprises samples of each demographic group.},  F3, along with the proposed heuristics outperform the state-of-the-art methods for fairness. 

To do so, we consider an oracle that makes the data homogeneous by distributing the samples such that each agent has access to all the demographic groups. Under the presence of such an oracle, we can locally train models for fairness using LMM~\cite{fairalm} at the client level. Further, we can also aggregate these models to obtain a global model, referred to as \texttt{FedAvg-LMM}. Likewise, under the presence of such an oracle, aggregating the local models trained only for accuracy leads to \texttt{FedAvg}. 
Table~\ref{table:oracle} presents the accuracy and fairness performance of our heuristics compared to \texttt{FedAvg} and \texttt{FedAvg-LMM}.
The accuracies observed for \texttt{FedAvg} and \texttt{FedAvg-LMM} on FairFace, FFHQ and UTK are in-line with~\cite{fairalm,fairGAN}.

%
\begin{table}[!t]
    \centering\scriptsize
\begin{tabular}{|c|c|c|c|c|c|}
 \hline
 \textbf{Dataset} & \textbf{Heuristic} & \textbf{Accuracy} & \textbf{$\Delta_{EOpp}$} & \textbf{$\Delta_{EO}$}& \textbf{$\Delta_{AP}$}\\
 \hline
 \multirow{3}{*}{FairFace~\cite{fairface}} & \texttt{FedAvg}~\cite{fedavg} &  79.5\%&  2.1&  15.1&  17.2\\
& \texttt{FedAvg-LMM}~\cite{fairalm} &  76.3\%  &  8.4& 3.0&  11.4\\
& 10-\textsf{FairAccAvg} &  73.3\% & 0.8&  9.6&  10.4\\
\hline
\multirow{3}{*}{FFHQ~\cite{ffhq}} & \texttt{FedAvg}~\cite{fedavg} &  90.1\%&  0.8&  19.1&  20.5 \\
& \texttt{FedAvg-LMM}~\cite{fairalm} & 86.6\%&  2.3&  6.4 &  8.7 \\
& 10-\textsf{FairAccAvg} & 82.7\%&  1.8&  2.0& 3.8\\
 \hline
\multirow{3}{*}{UTK~\cite{utk}} & \texttt{FedAvg}~\cite{fedavg} & 94.1\%  &  1.8& 11.3&  13.1\\
& \texttt{FedAvg-LMM}~\cite{fairalm} & 93.3\% &  3.4&  6.4& 9.8\\
& 10-\textsf{FairAccAvg} &  93.7\%&  3.1&  5.3&  6.8\\
\hline
\end{tabular}
\caption{Accuracy and Fairness Violation, $\Delta_k$, $k=\{$EOpp, EO, AP$\}$: we show the performance of our heuristic 10-\textsf{FairAccAvg} in comparison to \texttt{FedAvg}~\cite{fedavg} and \texttt{FedAvg-LMM}~\cite{fairalm}, with oracle access to homogeneous data. Notice that, 10-\textsf{FairAccAvg} outperforms \texttt{FedAvg-LMM} w.r.t. fairness on UTK~\cite{utk} while maintaining comparable accuracies with \texttt{FedAvg}.}
\label{table:oracle}
\end{table}

\noindent\underline{Comparison of results with Oracle:}  {From Table~\ref{table:oracle}, observe that 10-\textsf{FairAccAvg} ensures similar or significantly better fairness compared to \texttt{FedAvg-LMM}. For instance, $\Delta_{AP}$ with 10-\textsf{FairAccAvg} has 8\%, 56\% and 30\% improvement compared to \texttt{FedAvg-LMM}. 10-\textsf{FairAccAvg} also has accuracy comparable to \texttt{FedAvg}, the default approach for FL trained on homogeneous data. E.g., in UTK with 10-\textsf{FairAccAvg}, the accuracy only drops by 0.4\% compared to \texttt{FedAvg}, while the improvement in fairness for $\Delta_{EO}$ and $\Delta_{AP}$ is 53\% and 48\%, respectively.}

 In the Appendix, we also provide results on additional experiments, including (i) an ablation study for different hyper parameters ($\alpha$, $a$ and $\tau$), (ii) different network architecture, and (ii) experiments on accuracy and fairness trade-off for different FAC tasks and sensitive attributes. 

\smallskip
\noindent\textbf{Discussion:} Overall, our results show that \textsf{FairBest} provides the lowest fairness violation, while $\alpha$-\textsf{FairAccAvg} provides the most practical accuracy and fairness trade-off. As $\alpha$ increases, we observe not only an increase in accuracy but also an increase in fairness violation (Fig.~\ref{fig:Alpha-plot}). As a result, a practitioner can select an appropriate $\alpha$ value for the desired accuracy-fairness trade-off.

The key advantage of the proposed heuristics within the F3  framework is that they ensure an appropriate selection of a \emph{subset} of clients' local models, balancing fairness violation and accuracy loss in the process. 
In contrast, \texttt{FedAvg-DH}~\cite{fedavg} selects local models \emph{randomly} leading to high fairness violation.


\section{Conclusion}

In this paper, we focus on Fair Attribute Classification (FAC) in FL setting with data heterogeneity. We observe that existing approaches to ensure fairness in FL do not work in a heterogeneous setting due to the unavailability of demographic-specific data samples across clients. To address this, we propose F3, a novel FL framework to achieve fairness in FAC. With F3, we introduce four aggregation heuristics that ensure fairness while simultaneously maximizing the model’s accuracy. Experimentally, our heuristics outperform the default counterpart in FL on challenging benchmark face datasets. Our heuristics’ performance is competitive even with state-of-the-art approaches designed for the homogeneous setting. Overall, the results suggest that F3 helps strike a practical balance between fairness and accuracy for FAC.

\bibliographystyle{unsrt}  
\bibliography{fedfairface}  

\newpage
\begin{appendix}

\section{Training over Epochs}
\begin{figure}[!ht]
    \centering
    \includegraphics[width=\linewidth]{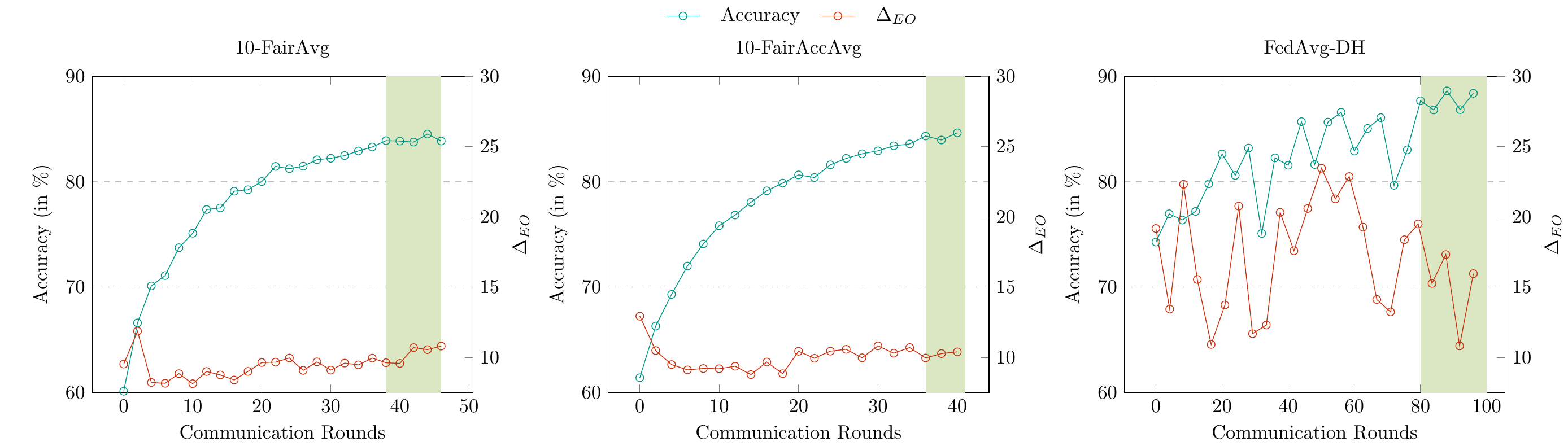}
    \caption{Accuracy and Fairness values over communication rounds for our heuristics \textsf{10-FairAvg} and \textsf{10-FairAccAvg} in comparison with \texttt{FedAvg-DH}. The shaded region depicts the rounds where training is stopped, i.e., the accuracy has not improved above 1\% over more than 10 rounds. Observe that in \texttt{FedAvg-DH} (plot 3) randomness in violation of fairness is high.}
    \label{fig:AvgEpochs1}
\end{figure}

\section{Effect of Changing Hyper Parameters}
Figure~\ref{fig5} shows the effect of changing $\alpha$ in $\alpha$-\textsf{FairAvg} (Eq. 2) and $\alpha$-\textsf{FairAccAvg} (Eq. 3)
    
\begin{figure}[!ht]
\centering
    \includegraphics[width=0.9\columnwidth]{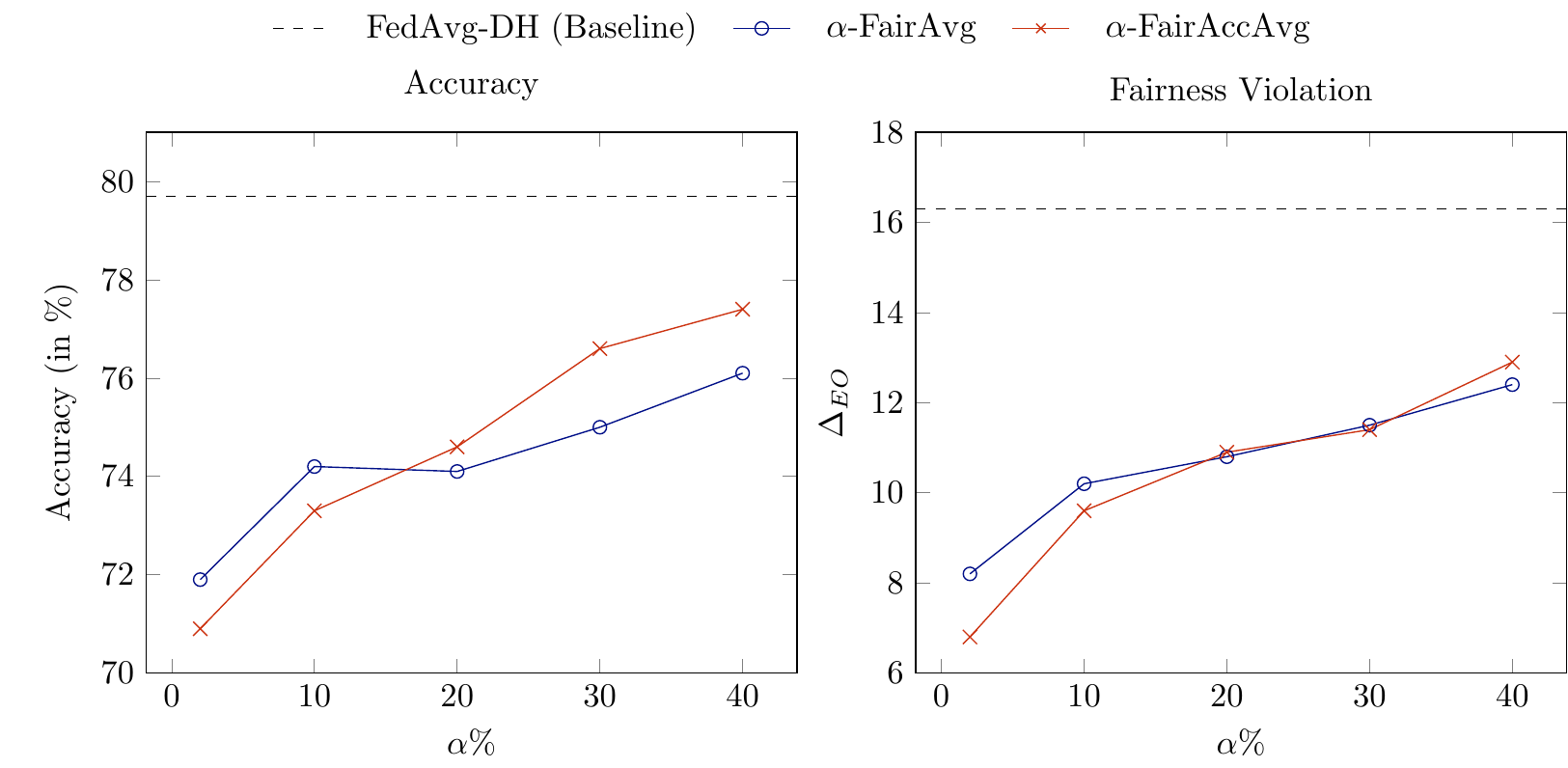}
    \caption{Accuracy and Fairness performance varying $\alpha$ in $\alpha$-FairAvg and $\alpha$-FairAccAvg (Fig. 5 in main paper)}
    \label{fig5}
\end{figure}

In Fig.~\ref{fig5}, we have seen that with increase in $\alpha$ our heuristics $\alpha$-\textsf{FairAvg} and $\alpha$-\textsf{FairAccAvg} tend to increase accuracy at a cost of fairness ($\Delta_{EO}$) mimicking standard method \texttt{FedAvg-DH}. We now give the performance of our heuristics w.r.t. remaining fairness notions.

    \begin{table}[!ht]
    \centering
    \adjustbox{max width=\columnwidth}{%
    \begin{tabular}{|c|c|c|c|c|c|c|}
    \hline
    \multirow{2}{*}{\textbf{Heuristic}} &
    \multirow{2}{*}{\textbf{$\alpha$}} & \multirow{2}{*}{\textbf{Accuracy} ($\uparrow$)} & \multicolumn{3}{c|}{Reduction in Fairness Violation ($\downarrow$)}\\
    \cline{4-6}
    & &  & \textbf{$\Delta_{EOpp}$} & \textbf{$\Delta_{EO}$}& \textbf{$\Delta_{AP}$}\\
    \hline
    \multirow{4}{*}{$\alpha$-\textsf{FairAvg}} 
     &  2 & 71.9\% & \cellcolor{green!25}{0.6}  & \cellcolor{green!25}{8.2} & \cellcolor{green!25}{8.8}\\
     & 10 & 74.2\%  & 3.3  & 10.1 & 13.5 \\
     & 20 & 74.1\% & 1.1  & 10.8 & 12\\
     & 30 & 75.0\%  & 2.1  & 11.5 & 17.2 \\
     & 40 & \cellcolor{magenta!25}{76.1\%}  & 1.8  & 12.4 & 22.6 \\
    \hline
    \multirow{4}{*}{$\alpha$-\textsf{FairAccRatio}} 
    & 2 & 70.9\%  & \cellcolor{green!25}{0.5}  & \cellcolor{green!25}{6.8} & \cellcolor{green!25}{7.4} \\
    & 10 & 73.3\%  & 0.8  & 9.6 & 10.4 \\
    & 20 & 74.6\%  & 2.5  & 10.9 & 13.4 \\
    & 30 & 76.6\%  & 2.7  & 11.4 & 14.1 \\
    & 40 & \cellcolor{magenta!25}{77.4\%}  & 1.4  & 12.9 & 14.3 \\
    \hline
    \end{tabular}}
    \caption{
    Accuracy and Fairness Violation, $\Delta_k$, $k=\{$EOpp, EO, AP$\}$: we show the performance of our novel heuristics $\alpha$-\textsf{FairAvg} and $\alpha$-\textsf{FairAccRatio}  on FairFace dataset~\cite{fairface} while varying the percentage ($\alpha$) of models chosen for aggregation. The highlighted cell provides the highest accuracy or least fairness. We notice that with increase in $\alpha$ the heuristics show improved accuracy at a cost of fairness.
    }
    \label{table:fairavg1}
    \end{table}
    
\section{Effect of Varying accuracy tolerance ($a$\%) and threshold rounds ($\tau$)}

    \begin{table}[!ht]
    \centering
    \adjustbox{max width=\columnwidth}{%
    \begin{tabular}{|c|c|c|c|c|c|c|}
    \hline
    \multirow{2}{*}{\textbf{Heuristic}} &
    \multirow{2}{*}{\textbf{$a$\% , $\tau$}} & \multirow{2}{*}{\textbf{Accuracy} ($\uparrow$)} & \multicolumn{3}{c|}{Reduction in Fairness Violation ($\downarrow$)}\\
    \cline{4-6}
    & & & \textbf{$\Delta_{EOpp}$} & \textbf{$\Delta_{EO}$}& \textbf{$\Delta_{AP}$}\\
    \hline
    \multirow{4}{*}{10-\textsf{FairAvg}} 
     & 1\%, 20 & 74.2\% & 3.3 & 10.2 & 13.5\\
     & 1\%, 10 & 73.9\%  & 1.4 & 10.0 &  11.4\\
     & 2\%, 20 & 74.0\% & 1.6 & 11.1& 12.7\\
     & 2\%, 10 & 72.5\%  & 1.1 & 9.4&  10.5\\
    \hline
    \multirow{4}{*}{10-\textsf{FairAccRatio}} 
     & 1\%, 20 & 73.3\% & 0.8 & 9.6 & 10.4\\
     & 1\%, 10 & 73.0\%  &  3.7 & 10.3& 14.0 \\
     & 2\%, 20 & 74.3\% & 3.3 & 9.0& 12.3\\
     & 2\%, 10 & 73.1\%  & 1.1 & 12.0&  13.1\\
    \hline
    \end{tabular}}
    \caption{
    Accuracy and Fairness Violation, $\Delta_k$, $k=\{$EOpp, EO, AP$\}$: we show the performance of our heuristics 10-\textsf{FairAvg} and 10-\textsf{FairAccRatio} on the FairFace dataset~\cite{fairface} varying the parameters accuracy tolerance $a$ and threshold $\tau$ (Algorithm 1). We observe similar performance in terms of the trade-off between accuracy and fairness violation, irrespective of $a$ and $\tau$. In the main paper, we used $a=1$\% and $\tau = 20$ for reporting the results.
    }
    \label{table:fairavg}
    \end{table}
    
\section{Varying the Architecture :Effect of Using VGG16 Architecture}
\begin{table}[!ht]
    \centering
    \adjustbox{max width=\columnwidth}{%
    \begin{tabular}{|c|c|c|c|c|c|}
    \hline
    \multirow{2}{*}{\textbf{Heuristic}} & \multirow{2}{*}{\textbf{Accuracy} ($\uparrow$)} & \multicolumn{3}{c|}{Reduction in Fairness Violation ($\downarrow$)}\\
    \cline{3-5}
    & & \textbf{$\Delta_{EOpp}$} & \textbf{$\Delta_{EO}$}& \textbf{$\Delta_{AP}$}\\
    \hline
        \multicolumn{5}{|c|}{ResNet18}\\
    \hline
    \texttt{FedAvg-DH} & 79.7\% & 1.0  & 16.3 & 17.3\\
     \textsf{FairBest} & 72.0\% & 0.6  & 8.2 & 8.8\\
     10-\textsf{FairAvg} & 74.2\%  & 3.3  & 10.2 & 13.5 \\
     10-\textsf{FairAccAvg} & 73.3\% & 0.8 & 9.6 & 10.4 \\
    \hline
    \multicolumn{5}{|c|}{VGG~}\\
    \hline
    \texttt{FedAvg-DH} & 76.0\% & 1.9  & 13.7 & 15.6\\
     \textsf{FairBest} & 68.5\% & 0.3  & 5.7 & 6.0\\
     $10$-\textsf{FairAvg} & 71.1\%  & 1.6  & 9.8 & 11.4 \\
     $10$-\textsf{FairAccAvg} & 72.9\% & 1.1 & 9.1 & 10.2 \\
    \hline
    \end{tabular}}
    \caption{
    Accuracy and Fairness Violation, $\Delta_k$, $k=\{$EOpp, EO, AP$\}$: we show the performance of our proposed heuristics with the baseline \texttt{FedAvg-DH}~\cite{fedavg} on the FaiFace dataset~\cite{fairface}, employing ResNet18 and VGG16 architectures. Notice that, the performance of VGG16 in terms of accuracy is significantly less than Resnet18's. Also, ResNet18 training is faster compared to VGG16. Hence, we use Resnet18 architecture for the rest of our experiments.
    }
    \label{table:vgg}
    \end{table}

\section{Accuracy and Fairness Trade-off using Mahalanobis Distance}
We observe the accuracy and fairness trade-off among our heuristics based on the metric \emph{Mahalanobis Distance}. We use the percentage loss in accuracy (Error in \%) and fairness violation ($\Delta_{k}$, $k = \{EOpp, EO, AP\}$) for this calculation. We compute Mahalanobis distance (MD) as follows,
    
    $$
        MD(\overline{x}) = \sqrt{(\overline{x} - \overline{\mu})^{T} S^{-1} (\overline{x} - \overline{\mu})}
    $$
    
    where, $\overline{x}=(Error, \Delta_{k})$ is vector with observed variables: Error and Fairness Violation ($\Delta_{k}$), $\overline{\mu}$ is vector with mean values of observed variables and $S$ is covariance matrix of observed variables.

Table~\ref{table:md} shows the results using mahalanobis distance
    
    \begin{table}[!ht]
    \centering
    \adjustbox{max width=\textwidth}{
        \begin{tabular}{|c|c|c|c|c|c|c|c|c|c|}
          \hline
          & \multicolumn{3}{c|}{{FairFace}~\cite{fairface}} & \multicolumn{3}{c|}{{FFHQ}~\cite{ffhq}} & \multicolumn{3}{c|}{{UTK}~\cite{utk}} \\
          \hline
         \multirow{2}{*}{{Heuristic}} &  
         \multicolumn{3}{c|}{Mahalanobis Distance (Error (in \%), $\Delta_{k}$)} & 
         \multicolumn{3}{c|}{Mahalanobis Distance (Error (in \%), $\Delta_{k}$)} & 
         \multicolumn{3}{c|}{Mahalanobis Distance (Error (in \%), $\Delta_{k}$)}\\
         \cline{5-7} \cline{2-4} \cline{8-10}
          &  \textbf{$\Delta_{EOpp}$} & \textbf{$\Delta_{EO}$}& \textbf{$\Delta_{AP}$} & \textbf{$\Delta_{EOpp}$} & \textbf{$\Delta_{EO}$}& \textbf{$\Delta_{AP}$} & \textbf{$\Delta_{EOpp}$} & \textbf{$\Delta_{EO}$}& \textbf{$\Delta_{AP}$}\\
         \hline
         \textsl{FairBest} &8.3 & 232.3 & \cellcolor{magenta!25}{35.9} & 7.9 & 8.1 & 8.8 & 3.8 & 4.6 & 4.5 \\
         \hline
         $10$-\textsl{FairAvg} &8.0 & \cellcolor{magenta!25}{230.3} & 37.8 & 9.0 & 6.3 & 7.8 & \cellcolor{magenta!25}{1.6} & 2.9 & 2.9 \\
         \hline
         $10$-{{\textsl{FairAccAvg}}} & \cellcolor{magenta!25}{7.9} & 232.4 & 36.0 & \cellcolor{magenta!25}{6.7} & \cellcolor{magenta!25}{5.9} & \cellcolor{magenta!25}{7.3} & 3.5 & \cellcolor{magenta!25}{2.6} & \cellcolor{magenta!25}{2.5} \\
        \hline
        \end{tabular}}
        \caption{
        Mahalanobis distance calculated between $\overline{x} = (\text{Error}, \Delta_{k})$ and (0,0) for our heuristics on all the three datasets from the Fig. 6 (main paper). Lower the distance better the accuracy and fairness trade-off. The highlighted heuristic provides the better accuracy and fairness trade-off out of our proposed heuristics. We observe that 10-\textsf{FairAccAvg} performs better than the other two heuristics for EOpp on FairFace, EO and AP on UTK, and for all three fairness notions on FFHQ dataset.
        }
        \label{table:md}
    \end{table}

\end{appendix}
\end{document}